%% file: arxiv-tevc-bare_jrnl_compsoc.tex
\documentclass{article}

\usepackage{url}            % simple URL typesetting
\usepackage{booktabs}       % professional-quality tables
\usepackage{amsfonts}       % blackboard math symbols
\usepackage{nicefrac}       % compact symbols for 1/2, etc.
\usepackage{microtype}      % microtypography
\usepackage{amssymb}
\usepackage{algorithm}
\usepackage{caption}
\usepackage{amsmath}
\usepackage{amsthm}
\usepackage{graphicx}
\usepackage{subfigure}
\usepackage{tabularx}
\usepackage{pifont}
\usepackage[noend]{algpseudocode}
\usepackage{bm}
\usepackage{array}
\usepackage{balance}
\usepackage{amsthm}
\usepackage{amsmath}
\usepackage{amssymb}
\usepackage{multirow}
\usepackage{fullpage}
\usepackage{color}

\usepackage[colorlinks=exact]{hyperref}       % hyperlinks
\usepackage{natbib}
\allowdisplaybreaks

\input{yaweinewcomm}

 \newtheorem{Definition}{\bf{Definition}}
 \newtheorem{Theorem}{\bf{Theorem}}
 
 \newtheorem{Lemma}{\bf{Lemma}}
 
 \newtheorem{Corollary}{\bf{Corollary}}
 
 \newtheorem{Assumption}{\bf{Assumption}}
 
 \newtheorem{Remark}{\bf{Remark}}

\date{}

\begin{document}

\title{Dynamic Online Gradient Descent with Improved Query Complexity: A Theoretical Revisit}

\author{Yawei Zhao$^1$, En Zhu$^1$, Xinwang Liu$^1$, and Jianping Yin$^2$ \\ 
$^1$National University of Defense Technology, Changsha, 410073, China. \\ 
$^2$Dongguan University of Technology, Dongguan, Guangdong, 523808, China. \\ 
E-mail: \{zhaoyawei, enzhu, xinwangliu\}@nudt.edu.cn; jpyin@dgut.edu.cn.
}

\maketitle

\begin{abstract}
We provide a new theoretical analysis framework to investigate online gradient descent in the dynamic environment. Comparing with the previous work, the new framework recovers the state-of-the-art dynamic regret, but does not require extra gradient queries for every iteration. Specifically, when functions are $\alpha$ strongly convex and $\beta$ smooth, to achieve the state-of-the-art dynamic regret, the previous work requires  $\Ocal(\kappa)$ with $\kappa = \frac{\beta}{\alpha}$ queries of gradients at every iteration. But, our framework shows that the query complexity can be improved to be $\Ocal(1)$, which does not depend on $\kappa$. The improvement is significant for ill-conditioned problems because that their objective function usually has a large $\kappa$. Those ill-conditioned problems wildely exist in image processing applications, e.g., image deblurring.

\end{abstract}

\section{Introduction}
Online Gradient Descent (OGD) has drawn much attention in the  community of machine learning \cite{Zhu:2015tr,Hazan2007Adaptive,Hall:2015ct,ShalevShwartz:2012dz,Garber:2018wf,Bedi:2018te}. It is widely used in various applications such as online recommendation \cite{Song:2008:RAT}, search ranking \cite{Moon:2010}. Generally, OGD is formulated as a game between a learner and an adversary. At the $t$-th round of the game, the learner submits $\x_t$ from the feasible set $\Xcal$, and the adversary selects a function $f_t:\Xcal \mapsto \RR$. Then, the function $f_t$ is returned to the learner, and incurs the loss $f_t(\x_t)$. 

Recently, there has been a surge of interest in analyzing OGD by using the dynamic regret \cite{Zinkevich:2003,Mokhtari:2016jz,Yang:2016ud,Lei:2017:CUO}. The dynamic regret is usually defined as 
\begin{align}
\label{eq:regret_dy}
R_T^{\ast} = \sum\limits_{t=1}^T f_t(\x_t) - \sum\limits_{t=1}^T f_t(\x_t^{\ast}), 
\end{align} where $\x_t^{\ast} \in \argmin_{\x\in\Xcal} f_t(\x)$. Unfortunately, it is well-known that a sublinear dynamic regret bound cannot be achieved in the worst case \cite{Zinkevich:2003}. The reason is that the functions $f_1, ..., f_T$ may be changed arbitrarily in the dynamic environment. But, it is possible to upper bound the dynamic regret in terms of certain regularity of the comparator sequence. Those regularities are usually defined as the \textit{path length} \cite{Mokhtari:2016jz,Yang:2016ud}: 
\begin{align}
\nonumber
\Pcal_T^{\ast} := \Pcal(\x_1^{\ast}, ..., \x_T^{\ast}) = \sum\limits_{t=2}^T \norm {\x_t^{\ast} - \x_{t-1}^{\ast}},
\end{align}  or \textit{squared path length} \cite{Zhang:2016wl}: 
\begin{align}
\nonumber
\Scal_T^{\ast} := \Scal(\x_1^{\ast}, ..., \x_T^{\ast}) = \sum\limits_{t=2}^T \norm {\x_t^{\ast} - \x_{t-1}^{\ast}}^2.
\end{align} They capture the cumulative Euclidean norm or the square of Euclidean norm of the difference between successive comparators. When all the functions $f_1, ..., f_T$ are $\alpha$-strongly convex and $\beta$-smooth, the dynamic regret is bounded by $\Ocal(\Pcal_T^{\ast})$ \cite{Mokhtari:2016jz}.  When the local variations are small, $\Scal_T^{\ast}$ is much smaller than $\Pcal_T^{\ast}$. Thus, the state-of-the-art dynamic regret of OGD is improved to be $\Ocal(\min\{\Pcal_T^{\ast}, \Scal_T^{\ast}\})$ \cite{Zhang:2016wl}. 

\begin{table}
\centering
\caption{ Our method OGD recovers the state-of-the-art regret with improved query complexity.   }
\begin{tabular}{c|c|c|c}
\hline 
Algo. & Obj.  type  & Dynamic regret & Avg. queries\tabularnewline
\hline
\hline 
\cite{Mokhtari:2016jz} & strongly convex  & $\Ocal(\Pcal_t^{\ast})$ & $\Ocal(1)$\tabularnewline \hline
\cite{Zhang:2016wl} & strongly convex  & $\Ocal( \min\{\Scal_T^{\ast}, \Pcal_t^{\ast}\} )$ & $\Ocal(\kappa)$\tabularnewline
\hline 
Ours & strongly convex  & $\Ocal( \min\{\Scal_T^{\ast}, \Pcal_t^{\ast}\} )$  & $\Ocal(1)$\tabularnewline
\hline 
\end{tabular}
\label{table_dynamic_regret}
\end{table}

But, to achieve the state-of-the-art dynamic regret, i.e., $\Ocal(\min\{\Pcal_T^{\ast}, \Scal_T^{\ast}\})$, the variant of OGD in \cite{Zhang:2016wl} has to query  $\Ocal(\kappa)$ gradients for every iteration. Here, $\kappa := \frac{\beta}{\alpha}$ represents the condition number for the $\beta$ smooth and $\alpha$ strongly convex objective function $f_t$. For a large $\kappa$, the extremely large query complexity makes it not practical in the online setting. In the paper, we investigate the basic online gradient descent, and provide a new theoretical analysis framework. \textbf{Using the new analysis framework, we show that the dynamic regret $\Ocal(\min\{\Pcal_T^{\ast}, \Scal_T^{\ast}\})$ can be achieved with $\Ocal(1)$, instead of  $\Ocal(\kappa)$ queries of gradients in \cite{Zhang:2016wl}.}  Main theoretical results are outlined in Table \ref{table_dynamic_regret} briefly.  

The improvement of the query complexity is vitally important for ill-conditioned\footnote{`ill-conditioned' may be notated by `ill-posed' or `badly posed' in some literatures.} problems \cite{Tarantola:2004:IPT,Hansen:2006,ill-posed-doi:10} whose objective function usually has a large condition number, i.e., $\kappa$. Let us take the image deblurring problem as an example \cite{Hansen:2006}.  Suppose we have a blurred image $\y$, which is modeled by using an unknown real image $\x$ and a blurring matrix $\A$. That is, $\y = \A \x$. Here, $\A$ is usually a non-singular matrix with a large condition number, e.g., $\kappa = 10^6$. We want to recover the real image $\x$ from the blurred image $\y$, that is, $\x = \A^{-1}\y$. Comparing with the method in \cite{Zhang:2016wl}, our new analysis framework shows that OGD is good enough, and the required queries of gradients can be reduced by multiple orders.

The paper is organized as follows. Section \ref{sect_related_work} reviews the related work. Section \ref{sect_preliminary} presents the preliminaries. Section \ref{sect_theoretical_analysis} presents our theoretical analysis framework. Section \ref{sect_dynamic_regret_strongly_convex} presents the improved bounds of regret and query complexity for the strongly convex case. Section \ref{sect_conclude} concludes the paper.

\section{Related work}
\label{sect_related_work}

\subsection{Regrets of OGD in the static environment.}
Online gradient descent in the static environment has been extensively investigated over the last ten years. The sublinear static regrets for smooth or strongly convex functions have been obtained in many literatures \cite{ShalevShwartz:2012dz,Hazan2016Introduction,Duchi:2011,Zinkevich:2003}. Specifically, when $f_t(\cdot)$ is strongly convex, the regret of online gradient descent is $\Ocal(\log T)$ \cite{Hazan2016Introduction}. When 
$f_t(\cdot)$ is convex but not strongly convex, the regret of online gradient descent is $\Ocal(\sqrt{T})$ \cite{Hazan2016Introduction}.

\subsection{Regrets of OGD in the dynamic environment.}
When all the functions $f_1, ..., f_T$ are $\alpha $ strongly-convex and $\beta$ smooth, the dynamic regret of OGD is $\Ocal(\Pcal_T^{\ast})$ \cite{Mokhtari:2016jz,Yang:2016ud}. If OGD queries $\Ocal(\kappa)$ gradients at every iteration, the dynamic regret of OGD can be improved to be $\Ocal(\min\{\Pcal_T^{\ast}, \Scal_T^{\ast}\})$ \cite{Zhang:2016wl}. But, our analysis framework shows that the $\Ocal(1)$ gradient queries for every iteration is enough to  obtain $\Ocal(\min\{\Pcal_T^{\ast}, \Scal_T^{\ast}\})$ dynamic regret.  Additionally,  there are some other regularities including the functional variation \cite{Zhu:2015tr,Besbes:2015gb} and  the gradient variation \cite{Chiang2012Online}. Those regularities measure different aspects of the variation in the dynamic environment. Since they are not comparable directly, some researchers consider to bound the dynamic regret by using the mixed regularity \cite{Jadbabaie:2015wg}. Extending our theoretical framework to different regularities is an interesting  avenue for future work.

Besides, the new proposed theoretical analysis framework is inspired by \cite{Joulani:2017un}. \cite{Joulani:2017un} provides a theoretical analysis framework in the static environment, but our theoretical analysis framework works in the dynamic environment. 

\section{Preliminaries}
\label{sect_preliminary}

\subsection{Notations and assumptions}
We use the following notation. 
\begin{itemize}
\item 
The bold lower-case letters, e.g., $\x$  represent vectors.  The normal letters, e.g., $\beta$ represent a scalar number.  %and the bold upper-case letters, e.g., $\A$ represent the matrices.
\item $\eta_t$ represents the learning rate of Algorithm \ref{algo_ogd} at  the $t$-th iteration, and $\eta_{\min} := \min\{\eta_1, ..., \eta_T\}$. 
\item 
The condition number $\kappa$ is defined by $\kappa := \frac{\beta}{\alpha}$ for any  $\beta$ smooth and  $\alpha$ strongly convex  function $f_t$.  
\item 
$\lrnorm{\cdot}$ represents the $l_2$ norm of a vector.  
\item 
$\Pi_\Xcal(\cdot)$ represents the projection to a set $\Xcal$. 
\item $\Xcal_t^{\ast} := \argmin_{\x\in\Xcal} f_t(\x)$ represents the minimizer set at the $t$-th iteration. 
\item
Bregman divergence $B_f(\x,\y)$ is defined by $ B_f(\x,\y) := f(\x) - f(\y) - \lrangle{\nabla f(\y), \x - \y}$ for any function $f$.  
\end{itemize}

In the paper, functions $\{f_t\}_{t=1}^T$ are assumed to be convex and $\beta$ smooth (defined as follows).  
\begin{Definition}[$\beta$ smoothness]
\label{definition_f_t_smooth}
A function $f: \Xcal\mapsto \RR$ is $\beta$ smooth, if, for any $\x\in\Xcal$ and $\y\in\Xcal$, we have $f(\y)\le f(\x)+\lrangle{\nabla f(\x), \y-\x} + \frac{\beta}{2}\lrnorm{\y - \x}^2$.
\end{Definition} If the function $f_t$ is $\beta$ smooth, according to the definition of the Bregman divergence, we have $B_{f_t}(\x,\y) \le \frac{\beta}{2}\lrnorm{\x - \y}^2$ holds for any $\x\in\Xcal$ and $\y\in\Xcal$. The other assumptions used in the paper are presented as follows.

\begin{Assumption}[$\alpha$ strong convexity]
\label{assumption_f_t_strongly_convex}
For any $t$, the function $f_t: \Xcal\mapsto \RR$ is $\alpha $ strongly convex. That is, for any $\x\in\Xcal$ and $\y\in\Xcal$, $f_t(\y)\ge f_t(\x)+\lrangle{\nabla f_t(\x), \y-\x} + \frac{\alpha}{2}\lrnorm{\y - \x}^2$.
\end{Assumption}

\begin{Assumption}[Boundedness of gradients]
\label{assumption_bounded_gradient}
We assume  $\lrnorm{\nabla f_t(\x_t)}^2 \le G$ for any $t$.
\end{Assumption}
\begin{Assumption}[Boundedness of the domain of $\x$]
\label{assumption_bounded_distance_x}
We assume  $\lrnorm{\x_t - \x_t^{\ast}}^2\le R$ for any $t$.
\end{Assumption}

The above assumptions, i.e., Assumptions \ref{assumption_f_t_strongly_convex}-\ref{assumption_bounded_distance_x}, are the basic assumptions, which are used widely in previous researches \cite{ShalevShwartz:2012dz,Hazan2016Introduction,Duchi:2011,Zinkevich:2003}. Additionally, we make the following assumption, which is used to model the dynamic environment.

The above assumptions, i.e., Assumptions \ref{assumption_f_t_strongly_convex}-\ref{assumption_bounded_distance_x}, are the basic assumptions, which are used widely in previous researches \cite{ShalevShwartz:2012dz,Hazan2016Introduction,Duchi:2011,Zinkevich:2003}. Additionally, we make the following assumption, which allows the environment to change within a range. It is a mild assumption for many tasks such as time-serise prediction \cite{pmlr-v49-kuznetsov16,pmlr-v30-Anava13}, traffic forecasting \cite{Buch:2011:RCV}, time-varying medical image analysis \cite{4658174,Lee2009-vettr}, online recommendation \cite{Chang:2017:SR}. 

\begin{Assumption}[Boundedness of variations in the dynamic environment.]
\label{assumption_bounded_variation}
Denote $\x_{0}^\ast = \x_1$. For any $0\le i\le T$ and $0\le j\le T$, when $\lrnorm{\x_{i+1}^{\ast} - \x_i^{\ast}} >0$ and $\lrnorm{\x_{j+1}^{\ast} - \x_j^{\ast}} > 0$, there exists a constant $V\ge 1$ such that $\lrnorm{\x_{i+1}^{\ast} - \x_i^{\ast}} \le V \lrnorm{\x_{j+1}^{\ast} - \x_j^{\ast}} $.
\end{Assumption}

\subsection{Algorithm}
Recall the algorithm of the OGD. At the $t$-th iteration, it submits $\x_t$, and receives the loss function $f_t(\x_t)$. Querying the gradient of $f_t(\x_t)$, it updates $\x_t$ by using the projected gradient descent method. The details are presented in Algorithm \ref{algo_ogd}. 

Comparing with the state-of-the-art method, i.e., Algorithm \ref{algo_omgd}, OGD only requires one query of gradient for every iteration, while Algorithm \ref{algo_omgd} requires $\frac{\kappa + 1}{2}$ queries of gradient. When $\kappa$ is large, the query complexity of Algorithm \ref{algo_omgd} is much higher than OGD. Comparing with OMGD, i.e., Algorithm \ref{algo_omgd}, our new theoretical analysis framework shows that \textbf{OGD is good enough to recover the state-of-the-art dyanmic regret yielded by OMGD, but it only leads to $\Ocal\lrincir{1}$ query of gradient, instead of $\Ocal\lrincir{\kappa}$ queries of gradient required by OMGD.}

\begin{algorithm}[!t]
   \caption{OGD: Online Gradient Descent.}
   \label{algo_ogd}
   \begin{algorithmic}[1]
   \Require The learning rate $\eta_t$ with $1\le t\le T$.
       \For {$t=1,2, ..., T$}
           \State Submit $\x_t\in\Xcal$ and receive the function $f_t$ with $f_t: \mathcal{X}\mapsto \mathbb{R}$.
            \State Query the gradient $\nabla f_t(\x_t)$ of $f_t$.  
           \State $\x_{t+1} = \Pi_{\Xcal}\left(  \x_t - \eta_t \nabla f_t(\x_t)  \right)$. 
       \EndFor
       \Return $\x_{T+1}$
   \end{algorithmic}
\end{algorithm}

\begin{algorithm}[!t]
   \caption{OMGD: Online Multiple Gradient Descent \cite{Zhang:2016wl}. }
   \label{algo_omgd}
   \begin{algorithmic}[1]
   \Require The learning rate $\eta_t$ with $1\le t\le T$.
       \For {$t=1,2, ..., T$}
           \State Submit $\x_t\in\Xcal$ and receive the function $f_t$ with $f_t: \mathcal{X}\mapsto \mathbb{R}$.
           \State $\z_t^{(1)} = \x_t$, and $K = \frac{\kappa + 1}{2}$.
           \For {$j = 1, 2, ..., K$}
            \State Query the gradient $\nabla f_t(\z_t^{(j)})$ of $f_t$.  
           \State $\z_t^{(j+1)} = \Pi_{\Xcal}\left(  \z_t^{(j)} - \eta_t \nabla f_t(\z_t^{(j)})  \right)$. 
           \EndFor
           \State $\x_{t+1} = \z_t^{(K+1)}$.
       \EndFor
       
       \Return $\x_{T+1}$
   \end{algorithmic}
\end{algorithm}

\section{A new theoretical analysis framework}
\label{sect_theoretical_analysis}

In the section, we first provide a modular analysis framework, which does not depend on the assumption on the functions. Then, equipped with the strongly convex assumption, it yields specific results. 

\subsection{High-level thought}

Our original goal is equivalent to investigate whether the basic OGD, i.e., Algorithm \ref{algo_ogd} can obtain the state-of-the-art dynamic regret, i.e., $\min\{\Pcal_T^{\ast}, \Scal_T^{\ast}\}$.  Using the divide-and-control strategy, we divide the dynamic regret of OGD into two parts. 
\begin{enumerate}
\item The first part, denoted by $R^{\mathrm{o}}_T$, is caused by the online setting in the dynamic environment. It does not depend on the strongly convex assumption on the function $f_t$. 
\item The second part, denoted by $R^{\mathrm{m}}_T$, is due to the projected gradient descent step in Algorithm \ref{algo_ogd}. It depends on the assumption on the function $f_t$ such as convexity or strong convexity. 
\end{enumerate}

In the paper, our first contribution is to provide an upper bound of $R^{\mathrm{o}}_T$ without the strongly convex assumption of $f_t$. Then, benefiting from the rich theoretical tools in the static optimization, we successfully bound $R^{\mathrm{m}}_T $ by using the strongly convex assumption of $f_t$.

\subsection{Meta framework}
Generally, the dynamic regret of OGD is bounded as follows. 

\begin{Theorem}
\label{theorem_high_level_regret}
For any $\eta_t >0$ in Algorithm \ref{algo_ogd}, the dynamic regret of OGD defined in \eqref{eq:regret_dy} is bounded by 
\begin{align}
\nonumber
R_T^{\ast} \le R^{\mathrm{o}}_T + R^{\mathrm{m}}_T
\end{align}  where \begin{align}
\nonumber
R^{\mathrm{o}}_T := \sum\limits_{t=1}^{T} \frac{1}{2\eta_t} \lrincir{  -  \lrnorm{\x_t^{\ast} - \x_{t+1}}^2   + \lrnorm { \x_t^{\ast} - \x_t }^2 }
\end{align} and 
\begin{align}
\nonumber
R^{\mathrm{m}}_T := & \sum\limits_{t=1}^{T} \frac{1}{\eta_t} \lrincir{ -B_{\eta_t f_t}(\x_t^{\ast}, \x_t) + \eta_t (f_t(\x_t) - f_t(\x_{t+1}))}  + \sum\limits_{t=1}^{T} \frac{1}{\eta_t} \lrincir{\frac{\beta \eta_t - 1}{2}\lrnorm{\x_{t+1} - \x_t}^2}.
\end{align}
\end{Theorem} 

In Theorem \ref{theorem_high_level_regret}, $R^{\mathrm{o}}_T$ represents the regret due to the online setting, and $R^{\mathrm{m}}_T$ represents the regret due to the projected gradient descent updating step in Algorithm \ref{algo_ogd}. 
\begin{Remark}
Note that the upper bound of $R^{\mathrm{m}}_T$ depends on the strongly convex assumption of the function $f_t$.  
\end{Remark}

\begin{Theorem}
\label{theorem_recurrsive_bound}
Use Assumption \ref{assumption_bounded_variation}, and set  $\eta_t>0$ in Algorithm \ref{algo_ogd}. Denote $\x_0^{\ast} = \x_1$ and $\eta_{\min} = \min\{\eta_1, ..., \eta_T\}$. For any  $0 < \rho \le 1$, the regret due to the online setting, i.e., $R^{\mathrm{o}}_T$ is bounded by 
\begin{align}
\nonumber
& R^{\mathrm{o}}_T   \le   \frac{1-\rho + 2\rho V}{2\eta_{\min}(1-\rho)} \Scal_T^{\ast} + \frac{1}{2\eta_1}\lrnorm{\x_1^{\ast} - \x_1}^2  + \frac{1}{2}\lrincir{\sum\limits_{t=1}^{T-1} \lrincir{\frac{1}{\eta_{t+1}} - \frac{1}{\eta_t}} \lrnorm { \x_{t+1}^{\ast} - \x_{t+1} }^2    }.
\end{align} 
\end{Theorem} 

\begin{Remark}
Note that this upper bound of $R^{\mathrm{o}}_T$ does not depend on the strongly convex assumption of the function $f_t$. It still holds for the convex function $f_t$. 
\end{Remark}

\begin{Lemma}[Appeared in Proposition $2$ in \cite{Mokhtari:2016jz}]
\label{lemma_linear_x}
Use Assumption \ref{assumption_f_t_strongly_convex}. Let $\v_{t+1} = \Pi_{\Xcal} \lrincir{\v_t - \eta_t \nabla f_t(\v_t)}$ and $\Xcal_t^{\ast} := \argmin_{\v\in\Xcal} f_t(\v)$. Denote $\kappa = \frac{\beta}{\alpha}$. If $\eta_t \le \frac{1}{\beta}$ and $\rho = \sqrt{\frac{\kappa - 1}{\kappa}}$, we have $\lrnorm{\v_{t+1} - \x_t^{\ast} }    \le  \rho    \lrnorm{\v_t - \x_t^{\ast} }$.
\end{Lemma} According to Lemma \ref{lemma_linear_x}, when $f_t$'s are strongly convex, $0< \rho< 1 $ (See  Lemma \ref{lemma_linear_x}). When $f_t$'s are just convex, $\rho = 1$ (that is, $\alpha \rightarrow 0$). Recall that $R^{\mathrm{m}}_T$ depends on the strongly convex assumption of $f_t$'s. Equipped by Lemma \ref{lemma_linear_x},  we find that as long as $R^{\mathrm{m}}_T$ is further bounded, we are able to provide an upper bound for the dynamic regret.

\section{Improved query complexity for strongly convex $f_t$} 
\label{sect_dynamic_regret_strongly_convex}

When all $f_t$'s are smooth and strongly convex, the dynamic regret of our method OGD is upper bounded by the following theorem.
\begin{Theorem}
\label{theorem_s_sc_regret}
Use Assumptions \ref{assumption_f_t_strongly_convex}, \ref{assumption_bounded_gradient}, \ref{assumption_bounded_distance_x} and \ref{assumption_bounded_variation}. Setting  $\eta_t = \eta = \frac{1}{2(\beta+\beta^2/\alpha)}$ in Algorithm \ref{algo_ogd}, and $\rho = \sqrt{\frac{\kappa - 1}{\kappa}} < 1$, we bound the dynamic regret of OGD as
\begin{align}
\nonumber     
R_T^{\ast} \le \min \{J_1, J_2\}, 
\end{align} where 
\begin{align}
\nonumber
J_1 = & \frac{(1-\rho + 2\rho V)\lrincir{ \beta + \frac{\beta^2}{\alpha} }}{1-\rho} \Scal_T^{\ast}   + \lrincir{ \beta + \frac{\beta^2}{\alpha} }\lrnorm{\x_1^{\ast} - \x_1}^2   + \frac{1}{ 2\lrincir{ \beta + \frac{\beta^2}{\alpha} } }\sum\limits_{t=1}^{T} \lrnorm{\nabla f_t(\x_t^{\ast})}^2 \\ \nonumber
\lesssim & \Scal_T^{\ast} + \sum\limits_{t=1}^{T} \lrnorm{\nabla f_t(\x_t^{\ast})}^2,
\end{align} and 
\begin{align}
\nonumber
J_2 = & \frac{G\lrnorm{\x_1-\x_1^{\ast}}}{1-\rho} \Pcal_T^{\ast} + \frac{G}{1-\rho} \lesssim  \Pcal_T^{\ast}.
\end{align}

\end{Theorem} 

\begin{Corollary}
Suppose $\sum\limits_{t=1}^{T} \lrnorm{\nabla f_t(\x_t^{\ast})}^2 = \Ocal \lrincir{\Scal_T^{\ast}} $. According to Theorem \ref{theorem_s_sc_regret}, the dynamic regret of OGD is bounded by 
\begin{align}
\nonumber
R_T^{\ast} \le \min \{J_1, J_2\} \lesssim \min \{\Pcal_T^{\ast}, \Scal_T^{\ast} \},
\end{align} where $J_1$ and $J_2$ are defined in Theorem \ref{theorem_s_sc_regret}.

\end{Corollary} 
\begin{proof}
Recall Assumption \ref{assumption_bounded_distance_x}, and we have $\lrnorm{\x_1^{\ast} - \x_1}^2  \le R$. When $\sum\limits_{t=1}^{T} \lrnorm{\nabla f_t(\x_t^{\ast})}^2 = \Ocal \lrincir{\Scal_T^{\ast}} $, we have $J_1 \lesssim \Scal_T^{\ast}$. Similarly, we have $J_2\le \frac{G\sqrt{R}}{1-\rho} \Pcal_T^{\ast} + \frac{G}{1-\rho} \lesssim \Pcal_T^{\ast}$. Thus, we finally obtain
\begin{align}
\nonumber
R_T^{\ast} \le \min \{J_1, J_2\} \lesssim \min\{\Pcal_T^{\ast}, \Scal_T^{\ast} \}.
\end{align} It completes the proof.

\end{proof}

Recall the previous method, i.e., Algorithm \ref{algo_omgd}. Its dynamic regret has been proved, and we present it as follows.
\begin{Lemma}[Appeared in Theorem $3$ and Corollary $4$ in \cite{Zhang:2016wl}.]
\label{lemma_previous_result_sc_regret}
Use Assumptions \ref{assumption_f_t_strongly_convex}, \ref{assumption_bounded_gradient}, and \ref{assumption_bounded_distance_x}, and choose $\eta_t \le \frac{1}{\beta}$ in Algorithm \ref{algo_omgd}. Denote the dynamic regret of Algorithm \ref{algo_omgd} by $\tilde{R}_T^{\ast}$. Then, for any constant $\sigma>0$, $\tilde{R}_T^{\ast}$ is bounded by 
\begin{align}
\nonumber
\tilde{R}_T^{\ast} \le \min\{J_3, J_4\},
\end{align} where 
\begin{align}
\nonumber
J_3 = & 2G \Pcal_T^{\ast} + 2G \lrnorm{\x_1 - \x_1^\ast} \lesssim \Pcal_T^{\ast}, \\ \nonumber
J_4 = & \frac{1}{2\sigma} \sum_{t=1}^T \lrnorm{\nabla f_t(\x_t^\ast)}^2 +(\beta + \sigma) \lrincir{ 2\Scal_T^{\ast}  +  \lrnorm{\x_1 - \x_1^\ast}^2 } \\ \nonumber
\lesssim & \Scal_T^{\ast} + \sum_{t=1}^T \lrnorm{\nabla f_t(\x_t^\ast)}^2.
\end{align} Furthermore, suppose $\sum\limits_{t=1}^{T} \lrnorm{\nabla f_t(\x_t^{\ast})}^2 = \Ocal \lrincir{\Scal_T^{\ast}} $, and we thus have $\tilde{R}_T^{\ast} \lesssim  \min \{ \Pcal_T^{\ast}, \Scal_T^{\ast} \}$.
\end{Lemma}

Comparing with Lemma \ref{lemma_previous_result_sc_regret}, our new result achieves the same bound of the regret. But, OGD, i.e., Algorithm \ref{algo_ogd}, only requires one query of gradient for every iteration, which does not depend on $\kappa$, and thus outperforms Algorithm \ref{algo_omgd} by reduing the query complexity significantly. The following remarks hightlight the advantages of our analysis framework.
\begin{Remark}
Our analysis framework achieves the state-of-the-art dynamic regret presented in \cite{Zhang:2016wl} with a constant factor, and outperforms the dynamic regret $\Ocal(\Pcal_T^{\ast})$ presented in \cite{Mokhtari:2016jz}. 
\end{Remark}
\begin{Remark}
Our analysis framework shows that $\Ocal(1)$ queries of gradients for every iteration is enough to achieve the state-of-the-art dynamic regret, but \cite{Zhang:2016wl} requires $\Ocal(\kappa)$ queries of gradients for every iteration.
\end{Remark}

\section{Conclusion}
\label{sect_conclude}
We provide a new theoretical analysis framework to analyze the regret and query complexity of OGD in the dynamic environment. Comparing with the previous work, our framework achieves the state-of-the-art dynamic regret, and improve the required queries of gradient to be $\Ocal(1)$.

\appendix

\section*{Proof of theorems.}

\textbf{Proof of Theorem \ref{theorem_high_level_regret}:}

\begin{proof}
\begin{align}
\nonumber
 R_T^{\ast}  =& \sum\limits_{t=1}^T \frac{1}{\eta_t} \lrincir{\eta_t f_t(\x_t) -  \eta_t f_t(\x_t^{\ast})} \\ \label{equa_meta_dynamic_regret}
=& \sum_{t=1}^T\frac{1}{\eta_t} \lrincir{ \underbrace{\langle \eta_t\nabla f_t(\x_t), \x_{t+1}-\x_t^{\ast}\rangle}_{I_1} -  B_{\eta_t f_t}(\x_t^{\ast}, \x_t) }  + \sum_{t=1}^T\frac{1}{\eta_t} \lrincir{ \underbrace{\langle \eta_t\nabla f_t(\x_t), \x_t-\x_{t+1} \rangle}_{I_2} }.
\end{align}

Now, we begin to bound $I_1$. According to Lemma \ref{lemma_I1_before}, we obtain
\begin{align}
\label{equa_upper_bound_I1}
I_1  \le  \frac{1}{2}\lrincir{-  \lrnorm{\x_t^{\ast} - \x_{t+1}}^2   + \lrnorm { \x_t^{\ast} - \x_t }^2 - \lrnorm{\x_{t+1} - \x_t}^2}.
\end{align}

After that, we begin to bound  $I_2$. 
\begin{align}
\nonumber
I_2 &= \langle \eta_t\nabla f_t(\x_t), \x_t-\x_{t+1} \rangle \\ \nonumber
&= \eta_t f_t(\x_t) - \eta_t f_t(\x_{t+1}) + \eta_t B_{f_t}(\x_{t+1}, \x_t)  \\  \label{equa_upper_bound_I2}
&\le \eta_t(f_t(\x_t) - f_t(\x_{t+1})) +  \frac{\beta \eta_t}{2}\lrnorm{\x_{t+1} - \x_t}^2.
\end{align} The last inequality holds because that all $f_t$'s are $\beta$ smooth. Substituting  (\ref{equa_upper_bound_I1}) and (\ref{equa_upper_bound_I2}) into (\ref{equa_meta_dynamic_regret}), we  finally complete the proof.
\end{proof}

\textbf{Proof of Theorem \ref{theorem_recurrsive_bound}:}
\begin{proof}
According to the cosine theorem, we have 
\begin{align}
\label{equa_cosine}
-  \lrnorm{\x_t^{\ast} - \x_{t+1}}^2   + \lrnorm { \x_{t+1}^{\ast} - \x_{t+1} }^2 \le  2 \lrnorm {\x_{t+1}^{\ast} - \x_t^{\ast}} \lrnorm{\x_{t+1} - \x_{t+1}^{\ast}} - \lrnorm{\x_{t+1}^{\ast} - \x_t^{\ast}}^2.
\end{align} 

According to Lemma \ref{lemma_linear_x}, if $f_t$ is convex and smooth,  $\lrnorm{\x_{t+1} - \x_t^{\ast}}  \le \rho \lrnorm{\x_{t} - \x_t^{\ast}}$ holds for $0<\rho\le 1$. Specifically, $0<\rho <1$ holds when $f_t$ is strongly convex, and $\rho=1$ holds when $f_t$ is just convex.  We thus have 
\begin{align}
\nonumber
2 \lrnorm {\x_{t+1}^{\ast} - \x_t^{\ast}} \lrnorm{\x_{t+1} - \x_{t+1}^{\ast}} - \lrnorm{\x_{t+1}^{\ast} - \x_t^{\ast}}^2 \ge -  \rho^2 \lrnorm{\x_{t} - \x_t^{\ast}}^2   + \lrnorm { \x_{t+1}^{\ast} - \x_{t+1} }^2.
\end{align} Let $A_{t+1} = \lrnorm{\x_{t+1} - \x_{t+1}^{\ast}}$, $M_{t+1} = \lrnorm{\x_{t+1}^{\ast} - \x_t^{\ast}}$, and we thus have 
\begin{align}
\nonumber
2A_{t+1}M_{t+1} - M_{t+1}^2 \ge A_{t+1}^2 - \rho^2 A_t^2,
\end{align} that is, $(A_{t+1} - M_{t+1})^2 \le \rho^2A_t^2$. Thus, we have 
\begin{align}
\nonumber
A_{t+1} - M_{t+1} &\le \rho A_t \\ \nonumber
\rho A_{t} - \rho M_{t} &\le \rho^2 A_{t-1} \\ \nonumber
\cdots &  \\ \nonumber
\rho^{t-1} A_{2} - \rho^{t-1} M_{2} &\le \rho^t A_{1}.
\end{align} Summing up, we obtain
\begin{align}
\nonumber
A_{t+1} \le &  \rho^{t}A_1 + \left( M_{t+1} + \rho M_{t} + ... + \rho^{t-1}M_2 \right) \\ \nonumber
= & \rho^{t}\lrnorm{\x_1 - \x_1^\ast} + \sum_{i=2}^{t+1} \rho^{t+1-i} \lrnorm{\x_i^{\ast} - \x_{i-1}^{\ast}} \\ \nonumber
\refabovecir{=}{\textcircled{1}} & \sum_{i=1}^{t+1} \rho^{t+1-i} \lrnorm{\x_i^{\ast} - \x_{i-1}^{\ast}}   \\ \label{equa_A_upper_bound}
= & \lrnorm{\x_{t+1}^{\ast} - \x_t^{\ast}}  +  \sum\limits_{i=1}^{t}\rho^i \lrnorm{\x_{t+1-i}^{\ast} - \x_{t-i}^{\ast}}. 
\end{align} $\textcircled{1}$ holds due to letting $\x_0^\ast = \x_1$.

Substituting  (\ref{equa_A_upper_bound}) into (\ref{equa_cosine}), we obtain,
\begin{align}
\nonumber
-  \lrnorm{\x_t^{\ast} - \x_{t+1}}^2   + \lrnorm { \x_{t+1}^{\ast} - \x_{t+1} }^2 \le &  2 \lrnorm {\x_{t+1}^{\ast} - \x_t^{\ast}} A_{t+1} - \lrnorm{\x_{t+1}^{\ast} - \x_t^{\ast}}^2  \\ \label{equa_theorem_2_temp0}
\le & \lrnorm {\x_{t+1}^{\ast} - \x_t^{\ast}}^2 + 2\lrnorm {\x_{t+1}^{\ast} - \x_t^{\ast}} \lrincir{\sum\limits_{i=1}^{t}\rho^i \lrnorm{\x_{t+1-i}^{\ast} - \x_{t-i}^{\ast}}}.
\end{align} 
\textbf{Case $1$. } When $\lrnorm {\x_{t+1}^{\ast} - \x_t^{\ast}} > 0$, according to \eqref{equa_theorem_2_temp0}, we have
\begin{align}
-  \lrnorm{\x_t^{\ast} - \x_{t+1}}^2   + \lrnorm { \x_{t+1}^{\ast} - \x_{t+1} }^2 \refabovecir{\le}{\textcircled{1}} & \lrnorm {\x_{t+1}^{\ast} - \x_t^{\ast}}^2 + 2V\lrnorm {\x_{t+1}^{\ast} - \x_t^{\ast}}^2 \lrincir{\sum\limits_{i=1}^{t}\rho^i } \\ \nonumber
\le & \lrnorm {\x_{t+1}^{\ast} - \x_t^{\ast}}^2 + \frac{2\rho V}{1-\rho}\lrnorm {\x_{t+1}^{\ast} - \x_t^{\ast}}^2 \\ \nonumber
= & \frac{1-\rho + 2\rho V}{1-\rho}\lrnorm {\x_{t+1}^{\ast} - \x_t^{\ast}}^2. 
\end{align}  $\textcircled{1}$ holds according to Assumption \ref{assumption_bounded_variation}. 

\textbf{Case $2$. } When $\lrnorm {\x_{t+1}^{\ast} - \x_t^{\ast}} = 0$, according to \eqref{equa_theorem_2_temp0},  we have
\begin{align}
-  \lrnorm{\x_t^{\ast} - \x_{t+1}}^2   + \lrnorm { \x_{t+1}^{\ast} - \x_{t+1} }^2 = 0 = \frac{1-\rho + 2\rho V}{1-\rho}\lrnorm {\x_{t+1}^{\ast} - \x_t^{\ast}}^2. 
\end{align} Combining Case $1$ and Case $2$, we obtain
\begin{align}
\label{equa_recursive_upper_bound}
 - \lrnorm{\x_t^{\ast} - \x_{t+1}}^2   + \lrnorm { \x_{t+1}^{\ast} - \x_{t+1} }^2 \le \frac{1-\rho + 2\rho V}{1-\rho}\lrnorm {\x_{t+1}^{\ast} - \x_t^{\ast}}^2. 
\end{align}

Thus, we obtain
\begin{align}
\nonumber
& 2R^{\mathrm{o}}_T = \sum\limits_{t=1}^{T} \frac{1}{\eta_t} \lrincir{-\lrnorm{\x_t^{\ast} - \x_{t+1}}^2   + \lrnorm { \x_t^{\ast} - \x_{t} }^2} \\ \nonumber
= & \sum\limits_{t=1}^{T-1} \frac{1}{\eta_t} \lrincir{-\lrnorm{\x_t^{\ast} - \x_{t+1}}^2   + \lrnorm { \x_{t+1}^{\ast} - \x_{t+1} }^2}  + \sum\limits_{t=1}^{T-1} \lrincir{\frac{1}{\eta_{t+1}} - \frac{1}{\eta_t}} \lrnorm { \x_{t+1}^{\ast} - \x_{t+1} }^2 \\ \nonumber 
& + \frac{1}{\eta_1} \lrnorm{\x_1^{\ast} - \x_1}^2 - \frac{1}{\eta_{T}}\lrnorm{\x_T^{\ast} - \x_{T+1}}^2 \\ \nonumber
\le & \sum\limits_{t=1}^{T-1} \frac{1}{\eta_t} \lrincir{-\lrnorm{\x_t^{\ast} - \x_{t+1}}^2   + \lrnorm { \x_{t+1}^{\ast} - \x_{t+1} }^2}  + \frac{1}{\eta_1} \lrnorm{\x_1^{\ast} - \x_1}^2 + \sum\limits_{t=1}^{T-1} \lrincir{\frac{1}{\eta_{t+1}} - \frac{1}{\eta_t}} \lrnorm { \x_{t+1}^{\ast} - \x_{t+1} }^2 \\ \nonumber
\refabovecir{\le}{\textcircled{1}} & \sum\limits_{t=1}^{T-1} \frac{1}{\eta_t} \lrincir{ \frac{1-\rho + 2\rho V}{1-\rho}\lrnorm {\x_{t+1}^{\ast} - \x_t^{\ast}}^2 }  + \frac{1}{\eta_1} \lrnorm{\x_1^{\ast} - \x_1}^2 + \sum\limits_{t=1}^{T-1} \lrincir{\frac{1}{\eta_{t+1}} - \frac{1}{\eta_t}} \lrnorm { \x_{t+1}^{\ast} - \x_{t+1} }^2 \\ \nonumber
\le & \frac{1-\rho + 2\rho V}{\eta_{\min}(1-\rho)} \Scal_T^{\ast} + \sum\limits_{t=1}^{T-1} \lrincir{\frac{1}{\eta_{t+1}} - \frac{1}{\eta_t}} \lrnorm { \x_{t+1}^{\ast} - \x_{t+1} }^2  + \frac{1}{\eta_1}\lrnorm{\x_1^{\ast} - \x_1}^2.
\end{align} Here, $\eta_{\min} = \min\{\eta_1, \eta_2 ,..., \eta_T\}$. $\textcircled{1}$ holds due to \eqref{equa_recursive_upper_bound}. Dividing $\frac{1}{2}$ on both sides, we complete the proof.
\end{proof}

\textbf{Proof of Theorem \ref{theorem_s_sc_regret}:}

\begin{proof}

When the function $f_t$ is $\alpha$ strongly convex, we have 
\begin{align}
\label{equa_strongly_convex}
B_{f_t}(\x_t^{\ast}, \x_t) \ge \frac{\alpha}{2}\lrnorm{\x_t^{\ast} - \x_t}^2.
\end{align}  Substituting (\ref{equa_strongly_convex}) into Theorem \ref{theorem_high_level_regret}, we obtain 
\begin{align}
\nonumber
& R_T^{\ast} \\ \nonumber 
\le &  \sum\limits_{t=1}^{T} \frac{1}{\eta_t} \lrincir{  - \frac{1}{2} \lrnorm{\x_t^{\ast} - \x_{t+1}}^2   + \frac{1-\alpha\eta_t}{2} \lrnorm { \x_t^{\ast} - \x_t }^2 }   +  \sum\limits_{t=1}^{T} \frac{1}{\eta_t} \lrincir{ \frac{\beta\eta_t-1}{2} \lrnorm{\x_{t+1} - \x_t}^2 } \\ \nonumber 
& + \sum\limits_{t=1}^{T} \frac{1}{\eta_t} \lrincir{\eta_t(f_t(\x_t) - f_t(\x_{t+1})) } \\ \nonumber
\refabovecir{\le}{\textcircled{1}} & \sum\limits_{t=1}^{T} \frac{1}{2\eta_t} \lrincir{ -  \lrnorm{\x_t^{\ast} - \x_{t+1}}^2   + \lrnorm { \x_t^{\ast} - \x_t }^2 } \\ \nonumber
& + \sum\limits_{t=1}^{T} \frac{\eta_t\left(\beta+\frac{1}{2\eta_t} + \frac{\beta^2}{\alpha}\right)-1}{2\eta_t} \lrnorm{\x_{t+1} - \x_t}^2   + \sum\limits_{t=1}^{T}  \eta_t \lrnorm{\nabla f_t(\x_t^{\ast})}^2  \\ \nonumber
\refabovecir{\le}{\textcircled{2}} & \sum\limits_{t=1}^{T} \frac{1}{2\eta_t} \lrincir{  -  \lrnorm{\x_t^{\ast} - \x_{t+1}}^2   + \lrnorm { \x_t^{\ast} - \x_t }^2 }  +  \sum\limits_{t=1}^{T}\eta_t\lrnorm{\nabla f_t(\x_t^{\ast})}^2  \\ \nonumber
\le &  \frac{(1-\rho + 2\rho V)\lrincir{ \beta + \frac{\beta^2}{\alpha} }}{1-\rho} \Scal_T^{\ast} + \lrincir{ \beta + \frac{\beta^2}{\alpha} }\lrnorm{\x_1^{\ast} - \x_1}^2   + \frac{1}{ 2\lrincir{ \beta + \frac{\beta^2}{\alpha} } }\sum\limits_{t=1}^{T} \lrnorm{\nabla f_t(\x_t^{\ast})}^2. 
\end{align} $\textcircled{1}$ holds due to \eqref{equa_f_upper_bound} in Lemma \ref{lemma_f_upper_bound} by setting $\theta_1 = \alpha$ and $\theta_2 = 2\eta_t$. $\textcircled{2}$ holds because of $\eta_t = \frac{1}{2\lrincir{\beta+\frac{\beta^2}{\alpha}}}$ for $1\le t\le T$.  The last inequality holds due to Theorem \ref{theorem_recurrsive_bound}. 

Combining Lemma \ref{lemma_strongly_convex_regret}, we finally complete the proof.

\end{proof}

\section*{ Proof of lemmas.}

\begin{Lemma}
\label{lemma_equivalent_update_proximal}
Denote $h(\x )  = \lrangle {\eta_t \nabla f_t(\x_t), \x} + \frac{1}{2}\lrnorm {\x - \x_t}^2$.  If $\x_{t+1} = \Pi_{\Xcal} \lrincir{ \x_t - \eta_t \nabla f_t(\x_t) }$, we have
\begin{align}
\nonumber
\x_{t+1} \in \Argmin_{\x\in\Xcal} h(\x).
\end{align}

\end{Lemma}
\begin{proof}
Consider the following convex optimization problem 
\begin{align}
\label{equa_lemma_equivalent_update_proximal_objective}
\min_{\x\in \Xcal} h(\x )  = \min_{\x\in \Xcal} \lrangle {\eta_t \nabla f_t(\x_t), \x} + \frac{1}{2}\lrnorm {\x - \x_t}^2
\end{align} Denote the optimum set is $\Xcal_t^{\ast}$, that is, for any $\x^\ast \in \Xcal_t^{\ast}$, $h(\x^\ast) = \min_{\x\in \Xcal} h(\x )$ holds. 

According to the first-order optimality condition \cite{Boyd:2004}, we have, for any $\z\in\Xcal$ and $\x^\ast\in\Xcal_t^\ast$,
\begin{align}
\nonumber
0 \le & \lrangle{ \nabla h(\x^\ast), \z - \x^\ast} \\ \label{equa_lemma_equivalent_update_proximal}
= & \lrangle{ \eta_t \nabla f_t(\x_t) + \x^\ast - \x_t, \z - \x^\ast}.
\end{align} 

Recall that $\x_{t+1} = \Pi_{\Xcal} \lrincir{ \x_t - \eta_t \nabla f_t(\x_t) }$.
Thus, we have
\begin{align}
\nonumber
& \lrangle{ \eta_t \nabla f_t(\x_t) + \x_{t+1} - \x_t, \z - \x_{t+1}} \\ \nonumber
= & \lrangle{ \Pi_{\Xcal} \lrincir{ \x_t - \eta_t \nabla f_t(\x_t) } - (\x_t - \eta_t \nabla f_t(\x_t)), \z - \x_{t+1}} \\ \nonumber
\ge & 0.
\end{align}  That is, $\x_{t+1}$ satisfies the first-order optimality condition of \eqref{equa_lemma_equivalent_update_proximal_objective}. It completes the proof.
\end{proof}

\begin{Lemma}
\label{lemma_I1_before} 
Use Assumption \ref{assumption_bounded_distance_x}.
For any minimizer $\x_t^{\ast} \in \Xcal_t^{\ast}$ and $\Xcal_t^{\ast} :=  \argmin_{\x\in\Xcal} f_t(\x)$, we have 
\begin{align}
\label{equa_lemma_I1_before}
2\lrangle {\eta_t \nabla f_t(\x_t),  
\x_{t+1} - \x_t^{\ast}} \le   -\lrnorm{\x_t^{\ast} - \x_{t+1}}^2 - \lrnorm{\x_{t+1} - \x_t}^2  + \lrnorm { \x_t^{\ast} - \x_t }^2.
\end{align}
\end{Lemma}
\begin{proof}
First, we construct an auxiliary function $h(\cdot )  = \lrangle {\eta_t \nabla f_t(\x_t), \cdot} + \frac{1}{2}\lrnorm {\cdot - \x_t}^2$. According to Lemma \ref{lemma_equivalent_update_proximal}, we have $\x_{t+1} \in \Argmin\limits_{\x\in \Xcal} h(\x)$. Let $\bar{\x} = \x_{t+1} + \tau (\x_t^{\ast} - \x_{t+1})$ with $\tau\in (0,1]$. 
\begin{align}
\nonumber
0 \le & h(\bar{\x}) - h(\x_{t+1}) \\ \label{equa_tau}
= & \lrangle {\eta_t \nabla f_t(\x_t), \tau (\x_t^{\ast} - 
\x_{t+1})} + \frac{1}{2}\lrnorm{\bar{\x}}^2 - \frac{1}{2}\lrnorm {\x_{t+1}}^2  + \tau \lrangle {\x_t, \x_{t+1} - \x_t^{\ast}}.
\end{align}  Dividing $\tau$ on both sides, we obtain
\begin{align}
\nonumber
0 \le &  \lrangle {\eta_t \nabla f_t(\x_t), \x_t^{\ast} - 
\x_{t+1}} + \frac{1}{2\tau}\left (\lrnorm{\bar{\x}}^2 - \lrnorm {\x_{t+1}}^2\right)  + \lrangle {\x_t, \x_{t+1} - \x_t^{\ast}}\\ \nonumber
\refabovecir{\le}{\textcircled{1}} & \lrangle {\eta_t \nabla f_t(\x_t), \x_t^{\ast} - 
\x_{t+1}}  + \lim_{\tau\rightarrow 0^{+}}\frac{1}{\tau}\left (\frac{1}{2}\lrnorm{\bar{\x}}^2 - \frac{1}{2}\lrnorm {\x_{t+1}}^2\right) + \lrangle {\x_t, \x_{t+1} - \x_t^{\ast}}\\ \nonumber
= & \lrangle {\eta_t \nabla f_t(\x_t), \x_t^{\ast} - 
\x_{t+1}}  + \lim_{\tau\rightarrow 0^{+}}\lrincir{ \frac{\tau}{2}\lrnorm{\x_t^{\ast} - \x_{t+1}}^2 + \lrangle{ \x_{t+1}, \x_t^\ast - \x_{t+1} } }  + \lrangle {\x_t, \x_{t+1} - \x_t^{\ast}}\\ \nonumber
= & \lrangle {\eta_t \nabla f_t(\x_t), \x_t^{\ast} - 
\x_{t+1}} + \lrangle{\x_{t+1}, \x_t^{\ast} - \x_{t+1}}  + \lrangle {\x_t, \x_{t+1} - \x_t^{\ast}}\\ \nonumber
= & \lrangle {\eta_t \nabla f_t(\x_t), \x_t^{\ast} - 
\x_{t+1}} - \frac{1}{2} \lrnorm{\x_t^{\ast} - \x_{t+1}}^2  - \frac{1}{2}\lrnorm{\x_{t+1} - \x_t}^2  + \frac{1}{2}\lrnorm { \x_t^{\ast} - \x_t }^2.
\end{align} $\textcircled{1}$ holds because that  (\ref{equa_tau}) holds for any $\tau \in (0,1]$. 
 Re-arranging the items, we prove the conclusion.
\end{proof}

\begin{Lemma}
\label{lemma_f_upper_bound} 

Suppose that all $f_t$'s are $\beta$ smooth. For any $\theta_1>0$, $\theta_2>0$ and any   minimizer $\x_t^{\ast} \in \argmin_{\x\in\Xcal}f_t(\x)$, we have 

\begin{align}
\label{equa_f_upper_bound}
& f_t(\x_t) - f_t(\x_{t+1})  \\ \nonumber
\le & \frac{\theta_1}{2}\lrnorm{\x_t - \x_t^{\ast}}^2 + \lrincir{\frac{\beta^2}{2\theta_1}+\frac{1}{2\theta_2}}\lrnorm{\x_{t+1} - \x_t}^2  +\frac{\theta_2}{2}\lrnorm{\nabla f_t(\x_t^{\ast})}^2.
\end{align}
\end{Lemma}
\begin{proof}
For any $\theta_1>0$ and $\theta_2>0$, we have
\begin{align}
\nonumber
&f_t(\x_t) - f_t(\x_{t+1}) -\frac{\theta_2}{2}\lrnorm{\nabla f_t(\x_t^{\ast})}^2 - \frac{1}{2\theta_2} \lrnorm{\x_t - \x_{t+1}}^2\\ \nonumber
\le & \lrangle{\nabla f_t(\x_t), \x_t - \x_{t+1}} - \lrangle{\nabla f_t(\x_t^{\ast}), \x_t - \x_{t+1}}  \\ \nonumber
\le & \frac{\theta_1}{2\beta^2}\lrnorm{\nabla f_t(\x_t) - \nabla f_t(\x_t^{\ast})}^2 + \frac{\beta^2}{2\theta_1}\lrnorm{\x_{t+1} - \x_t}^2 \\ \nonumber
\le & \frac{\theta_1}{2}\lrnorm{\x_t - \x_t^{\ast}}^2 + \frac{\beta^2}{2\theta_1}\lrnorm{\x_{t+1} - \x_t}^2.
\end{align} The last inequality holds because that all $f_t$'s are $\beta$ smooth. Re-arranging the items, we thus complete the proof.
\end{proof}

\begin{Lemma}[Appeared in Theorem $1$ in \cite{Mokhtari:2016jz} ]
\label{lemma_strongly_convex_regret}
Suppose that Assumptions \ref{assumption_f_t_strongly_convex}-\ref{assumption_bounded_distance_x} hold, and all $f_t$ are $\beta$ smooth. Thus, $\rho = \sqrt{\frac{\kappa-1}{\kappa}} <1$ with $\kappa := \frac{\beta}{\alpha}$. Set $\eta_t \le \frac{1}{L}$ in OGD, i.e., Algorithm \ref{algo_ogd}. The dynamic regret of OGD is bounded as
\begin{align}
\nonumber
R_T^{\ast} \le \frac{G\lrnorm{\x_1-\x_1^{\ast}}}{1-\rho} \Pcal_T^{\ast} + \frac{G}{1-\rho}.
\end{align}

\end{Lemma}

\section*{Acknowledgment}

This work was supported by the National Key R \& D Program of China 2018YFB1003203 and the National Natural Science Foundation of China (Grant No. 61672528, 61773392, and 61671463). 


\balance

\bibliographystyle{abbrvnat}
\bibliography{reference}

\end{document}

%% file: yaweinewcomm.tex
\renewcommand{\v}{{\bf v}}

\newcommand{\x}{{\bf x}}
\newcommand{\y}{{\bf y}}
\newcommand{\z}{{\bf z}}

\newcommand{\A}{{\bf A}}

\newcommand{\Scal}{{\mathcal{S}}}

\newcommand{\Ocal}{{\mathcal{O}}}
\newcommand{\Pcal}{{\mathcal{P}}}

\newcommand{\Xcal}{{\mathcal{X}}}

\newcommand{\argmin}{\operatornamewithlimits{argmin}}
\newcommand{\Argmin}{\operatornamewithlimits{Argmin}}

\newcommand{\lrincir}[1]{\left( #1 \right)}

\newcommand{\norm}[1]{\lVert#1\rVert}
\newcommand{\lrnorm}[1]{\left\lVert#1\right\rVert}
\newcommand{\lrangle}[1]{\left\langle#1 \right\rangle}

\newcommand{\RR}{\mathbb{R}}

\newcommand{\refabovecir}[2]{\displaystyle_{#1}^{#2}}